\documentclass[conference]{IEEEtran}
\IEEEoverridecommandlockouts

\pdfoutput=1

\usepackage{amsmath,amssymb,amsfonts}
\usepackage{graphicx}
\usepackage{graphics}
\usepackage{textcomp}
\usepackage{stmaryrd}
\usepackage{bigints}

\usepackage{color}

\usepackage{algorithm}
\usepackage{algorithmic}
\usepackage{enumitem}

\usepackage{dsfont}
\usepackage{upgreek}
\usepackage{bbm}
\usepackage{mathbbol}
\usepackage{epstopdf}
\usepackage{tikz}
\usepackage{float}
\usepackage{booktabs}
\usepackage{subfigure}

\usepackage{balance}
\usepackage{xspace}
\usepackage{verbatim}
\usepackage[scaled=.7]{beramono}

\usepackage{amsthm}
\newtheorem{theorem}{Theorem}

\newtheorem{definition}{Definition}

\newtheorem{assumption}{Assumption}

\graphicspath{ {./} }

\newcommand{\Sec}[1]		{Sec.\,\ref{#1}}
\newcommand{\Fig}[1]		{Fig.\,\ref{#1}}

\newcommand{\Eq}[1]			{Eq.\,\ref{#1}}

\newcommand{\Theorem}[1]{Theorem~\ref{#1}}

\newcommand{\Definition}[1]{Definition~\ref{#1}}

\newcommand{\Experiment}[1]{Exp.\,#1}
\newcommand{\ie}   			{i.e.\xspace}
\newcommand{\eg}   			{e.g.\xspace}
\newcommand{\etc}   		{etc.\xspace}

\newcommand{\iid}   		{i.i.d.\xspace}

\newcommand{\Exp}    		{\mathbbm{E}}
\newcommand{\real}      {\mathbbm{R}}
\newcommand{\Prob}      {\mathbbm{P}}

\newcommand{\mydots} 	{...}

\newcommand{\inlinetitle}[2]  {\noindent\textbf{\emph{#1}{#2}}}

\newcommand{\International} {Intern.\xspace}
\newcommand{\Journal} 			{J.\xspace}
\newcommand{\Conference} 		{Conf.\xspace}
\newcommand{\Symposium} 		{Symp.\xspace}
\newcommand{\Transactions} 	{Trans.\xspace}
\newcommand{\Proceedings} 	{Proc.\xspace}

\def\x{{\mathbf x}}
\def\S{{\mathcal S}}

\newcommand{\MyTitle} {A Probabilistic Framework to Node-level Anomaly Detection in Communication Networks}

\begin{document}

\title{\MyTitle}
\author{\IEEEauthorblockN{Batiste Le Bars\IEEEauthorrefmark{1}\IEEEauthorrefmark{2} \qquad Argyris Kalogeratos\IEEEauthorrefmark{1}}
\vspace{5pt}
\IEEEauthorblockA{\IEEEauthorrefmark{1}Center of Applied Maths, ENS Cachan, CNRS, University Paris-Saclay, France\\
\IEEEauthorrefmark{2}Sigfox R\&D, Paris, France}
\vspace{5pt}
Emails: \texttt{\{lebars,\, kalogeratos\}@cmla.ens-cachan.fr}
}

\maketitle

\begin{abstract}
In this paper we consider the task of detecting abnormal communication volume occurring at node-level in communication networks. The signal of the communication activity is modeled by means of a \emph{clique stream}: each occurring communication event is instantaneous and activates an undirected subgraph spanning over a set of equally participating nodes. We present a probabilistic framework to model and assess the communication volume observed at any single node. Specifically, we employ non-parametric regression to learn the probability that a node takes part in a certain event knowing the set of other nodes that are involved. On the top of that, we present a concentration inequality around the estimated volume of events in which a node could participate, which in turn allows us to build an efficient and interpretable anomaly scoring function. Finally, the superior performance of the proposed approach is empirically demonstrated in real-world sensor network data, as well as using synthetic communication activity that is in accordance with that latter setting.
\end{abstract}


\section{Introduction}
Monitoring the activity in communication networks has become a popular area of research and particular attention has been paid to detection tasks such as spotting events or anomalies. An effective way to represent the communication activity is via a dynamic graph where the entities are considered to be nodes, and each communication event (or more simply \emph{event}) to be represented by a set of connecting edges that appear at a specific time interval. Multiple occurring events over time may be seen as a \emph{link stream} \cite{latapy2017stream} with fast creation and deletion of edges. The use of this representation is mainly motivated by the fact that, in reality, content-specific features of the communicated messages are usually kept undisclosed so as to preserve privacy. Consequently, most studies on activity monitoring merely deal with linkage information, \ie who communicated with whom and at which time; the body of work on anomaly detection is not an exception.

The anomaly detection task on graph-related activity can refer to the {node-,} the subgraph-, or the whole graph-level \cite{ranshous2015anomaly}. To the best of our knowledge, the existing methods consider \emph{time-aggregated} representations of the dynamic graph. It has been proposed to work with time-series of static graphs, each of them summarizing the link stream during a time interval. In other words, each edge weight of a static graph is a function of the number of events occurring between two nodes during that time interval. Modeling the weights' evolution with counting processes \cite{heard2010bayesian,corneli2017multiple} is among the standard approaches. The main drawback of any aggregated representation is that it neglects events that involve more than two nodes (\eg multiple receivers). Besides, a common limitation of the existing literature is the assumption that the communication volume is generated by a stationary underlying distribution.

In this work, we focus on the detection of \emph{abnormal communication volume at node-level}, which is particularly interesting as a change in the behavior of a node may reveal various types of abnormality (\eg account hack, antenna breakdown, \etc). 
We put forward a \emph{content-agnostic} approach supposing access solely to the linkage information observed at each event, that is the set of the involved nodes. The conceptual novelty of our approach is that, contrarily to our predecessors that use time-aggregated representations, we model the activity as a \emph{clique stream}. We track each event independently, we consider it to be instantaneous and thus to activate an undirected subgraph spanning over the set of equally participating nodes, \ie there are no special roles such as sender and receivers. Hence, we can represent each event with a binary \emph{fingerprint} indicating the involved nodes. Subsequently, we propose to statistically model and infer the probability that a node takes part in an event, knowing the observed event fingerprint indicating the other participating nodes. The assumption is that there is a pattern in the fingerprints of the events in which a node participates. This pattern results from the underlying network structure since it is natural for subsets of neighboring nodes to participate frequently together in events. 

This modeling allows us to derive confidence levels for the communication volume to which a node participates in a time interval. Our detection approach has two strong aspects. First, it allows the time-series of the node's communication volume to be non-stationary, since it only assumes regularity in the corresponding event fingerprints. Specifically, knowing the fingerprint of an event for all nodes but for a reference node, then the conditional probability that this node takes part in that event is constant over time, whereas the marginal probability that the node could participate in the event is not necessarily constant. Second, the anomaly score that our approach outputs is easily interpretable as it is simply based on the prediction error of a regression function.

\section{Related Work}\label{sec:related_works}

In the literature, the existing detection methods for abnormal node communication volume mostly analyze a time-aggregated representation of the actual dynamic graph of communication activity. This implies a time-series of static graphs $\{A_t\}_{t=1}^{T}$, where $A_t \in \real^{N\times N}$ is the weighted adjacency matrix representing all the shared communication events between pairs of nodes at in the time interval $t\in\{1,\mydots,T\}$, and $N$ is the total number of nodes in the network. Most methods do not consider self-edges and therefore require each $A_t$ to have zero diagonal. Since these methods consider node-to-node communication events, note that the weighted degree of a node according to $A_t$ gives also the total number of events observed at a node in the time interval $t$. The multivariate time-series of the total number of events occurring in the network over time can be written as $\{M_t\}_{t=1}^{T}$. This is the variable of our interest which we would like to know when it gets abnormal values.

A feature-based approach for detecting anomalies in such time-series of graphs, is to compute several graph features for each $A_t$, such as the node degree or centrality, and then apply standard anomaly detection techniques on the derived multivariate time-series of these features \cite{cheng2009detection,gupta2014outlier,chandola2009anomaly}. 
More generally, the literature of anomaly detection in time-series of graphs varies in three aspects:
\begin{itemize}
\item \emph{Availability of data labels}: Semi-supervised (access to a dataset of \emph{normal} system operation) \cite{scholkopf2001estimating,scott2006learning,DI20101910} or unsupervised (no label available) \cite{breunig2000lof, huang2007network}.
\item \emph{Type of the utilized method}: Probabilistic model-based \cite{priebe2005scan,wang2014locality,peel2015detecting,corneli2017multiple,aggarwal2011outlier,neil2013scan,heard2010bayesian}, distance based \cite{breunig2000lof}, decomposition-based \cite{sun2007less,kolda2008scalable}, compression-based \cite{huang2007network}, \etc
\item \emph{Scale of abnormality}: Node/Edge-level \cite{ji2013incremental,heard2010bayesian}, subgraph-level \cite{neil2013scan}, or whole network-level \cite{corneli2017multiple}.
\end{itemize}
The reader should refer to \cite{ranshous2015anomaly} for a more detailed survey on anomaly detection in dynamic graphs.

As in \cite{heard2010bayesian,pincombe2005anomaly,wan2009link}, our work assumes a semi-supervised setting and proposes a model-based approach for node-level anomaly detection. Moreover, as in \cite{corneli2017multiple}, graph edges are considered to be undirected, and each event to be shared by two or more nodes without distinguishable roles (\eg sender and receivers).

In \cite{latapy2017stream}, the \emph{link stream} framework is presented for the representation of a dynamic graph as a stream where edges are being created and removed. Therein, the nodes are assumed to be fixed and the dynamics affect only the edges between them. An edge is characterized by a triplet $(S,u,v)$ noting two communicating nodes $u$, $v$, and a time interval $S$ which is not necessarily continuous (may even be a union of non-contiguous time intervals). In this work, we adopt this stream framework. In particular, as we will see, we represent the activity as a clique stream, and $S$ is always a finite union of singletons as edges appear instantaneously. 

\section{Model Description and Methodology}\label{sec:model}

\subsection{The model}\label{sec:the_model}

Let a communication network have $N$ inter-connected entities, referred to as nodes. 
In terms of notation style, we differentiate a random from an observed variable (respectively vector) with uppercase and lowercase (respectively bold) letters. Moreover, let $|\cdot|$ denote the size of the input set.

\begin{definition} \label{def:event}
(Communication event): A communication event $e=(\tau_e,X_e)$ is denoted by a tuple of $N+1$ elements, which contains the timestamp $\tau_e$ at which the event occurred and its fingerprint $X_e$.
\end{definition} 

\begin{definition} \label{def:fingerprint}
(Event fingerprint): The fingerprint of an event $e$ is an $N$-dimensional binary vector $X_e\in \{0,1\}^N$, where $X_e^{(j)}=1$ if node $j$ is involved in the event, and $0$ otherwise.
\end{definition} 

Note that the involvement of a node in an event implies its participation regardless its communication role (\eg sender or receiver). From a probabilistic point of view, a fingerprint follows a multivariate Bernoulli distribution. From a graph point of view, we can see that each event creates a \emph{clique} with all the involved nodes (see an example in \Fig{fig:g}). Formally, a clique is defined as a subset of nodes of the graph that are all pairwise adjacent. Therefore, we regard the communication activity as a \emph{clique stream}, and each clique appears instantaneously as events have no duration.

\begin{definition} \label{def:linkStream}
(Event stream): An event stream $\S=\{(\tau_s,X_s)\}_{s=1}^{n}$ is a sequence of $n\triangleq|\S|$ communication events each creating a clique among the involved nodes. We write as $\S_t\subset \S$ the sub-stream with the events that occurred in a certain time interval $t$, and $n_t\triangleq|\S_t|$.
\end{definition} 

\begin{assumption}
The communication events are considered to be independent. The total number of events $n$ recorded during an event stream $\S$ is considered to be deterministic.
\end{assumption}

Let us consider a time interval $t$ and the associated event stream $\S_t$ consisting of its $n_t$ recorded events. Let the event realizations be denoted by $\{\x_i\}_{i=1}^{n_t}$, where $\forall i, \x_i\in \{0,1\}^N$. Also, let $M_t^{(j)}= \sum_{i=1}^{n_t}X_i^{(j)}$ be the number of events recorded at node $j\in \{1,\mydots,N\}$ over the time interval $t$. 

For a given node $j$ and time interval $t$, the goal of our method is to be able to decide if the volume of events in which that node participates is abnormal. To solve this problem, the main idea is to provide confidence levels for $M_t^{(j)}$ based on the fingerprints collected from events of the neighboring nodes. This way, an anomaly can be simply spotted whenever the observed value of $M_t^{(j)}$ lies out of the confidence level.

\begin{definition} \label{def:cpf}
(Conditional probability function): 
Let $\x^{(-j)}$ be the fingerprint of the event $X$ that indicates the participation of all nodes except from node $j$. Then, we define as $\eta^*_j(\x^{(-j)})$ the probability that node $j$ participates in the event $X$, provided the fingerprint $\x^{(-j)}$:

\begin{align}
\eta^*_j(\x^{(-j)})&\triangleq\Prob(X^{(j)}=1|X^{(-j)}=\x^{(-j)}) \\
&=\Exp\left[X^{(j)}|X^{(-j)}=\x^{(-j)}\right].
\end{align} 
\end{definition} 
Knowing the fingerprint over all the other nodes allows us to express the behavior of node $j$ as a Bernoulli random variable: 
\begin{equation}\label{eq:XjBernoulli}
X^{(j)}\sim \mathcal{B}\left(\eta^*_j(\x_i^{(-j)})\right).
\end{equation}

Concerning a sub-stream $\S_t$ and the number of events recorded at node $j$ therein, we can note that $M_t^{(j)}$ is a sum of Bernoulli distributions and, thus, we can use concentration inequalities \cite{boucheron2013concentration}, such as Chernoff's or Hoeffding's \cite{hoeffding1963probability}, to derive confidence levels. Below, we apply the \emph{bilateral Hoeffding's inequality} to our case:
\begin{equation}
\label{eq:hoeffding}
\begin{split}
\!\!\!\!\!\Prob \Bigg(\left|M_t^{(j)}-\mu^*\right| \geq \varepsilon \ \ \Bigg| \ \ \forall i=1\ldots n_t &, \quad X_i^{(-j)}=\x_i^{(-j)} \Bigg) \\
&\ \ \, \leq 2 \exp \left(-\frac{2\varepsilon^2}{n_t}\right)
\end{split}
\end{equation}
with $\mu^*=\Exp\left[M_t^{(j)} | X^{(-j)} = \x^{(-j)}\right] = \sum_{i=1}^{n_t}\eta^*_j(\x_i^{(-j)})$.
Using this inequality and, as mentioned earlier, knowing the event fingerprints of all other nodes, we have with probability at least $1-\delta$:
\begin{equation}\label{eq:chernoff}
\left|M_t^{(j)}-\mu^*\right| \leq \sqrt{\frac{n_t\log(2/\delta)}{2}}.
\end{equation} 
This equation provides, with high probability (as $\delta$ is close to $0$), a good confidence interval for $M_t^{(-j)}$.

\subsection{Methodology}\label{subsec:methodo}

Suppose we observe the sub-stream $\S_t$ and the associated $n_t$ fingerprints. Let $x_i^{(j)}$ be the observed version of $X_i^{(j)}$ indicating if node $j$ participated in the event $i$ or not, and $m_t^{(j)}=\sum_{i=1}^{n_t}x_i^{(j)}$ be the observed version of $M_t^{(j)}$. If we suppose the access to $\eta_j^*$ (\ie the true conditional probability for node $j$), then an intuitive \emph{anomaly score} for the $m_t^{(j)}$ is:
\begin{equation}\label{eq:anomaly_score}
\rho_t^j = 2\,\exp\left(-\frac{2\,(m_t^{(j)}-\mu^*)^2}{n_t}\right)\!.
\end{equation}
This score is obtained by replacing $\varepsilon$ with $(m_t^{(j)}-\mu^*)$ in the right-hand side of \Eq{eq:hoeffding}. Relating to the statistical hypothesis testing theory, this score can be seen as an upper bound on the {$p$-value}. Then, a threshold $\alpha$ can be set, conventionally $0 < \alpha \leq 0.05$, to detect anomalies. More specifically, an anomaly is detected when $\rho_t^j < \alpha$. Note that this method is equivalent to replacing $\delta$ with the chosen threshold value, and then checking if $(m_t^{(j)}-\mu^*)$ falls out of the confidence interval in \Eq{eq:chernoff}. Bear also in mind that, since this method provides an upper bound on the {$p$-value}, it is in fact more conservative than in the standard statistical testing. Indeed, the confidence intervals built with \Eq{eq:chernoff} are larger than the confidence intervals that correspond to probability exactly equal to $\alpha$. 

In practice, we cannot have access to the true conditional probability functions $\eta^*_j(\cdot)$ which need to be estimated. To this end, we suppose that we have access to a training data stream $\S_0=\{(\tau_i^0,X^0_i)\}_{i=1}^{n_0}$ which is an event stream recorded at times of normal communication behavior for all nodes. With our definition of $\eta^*_j(\cdot)$ (\Definition{def:cpf}), the estimation problem refers to the task of estimation of conditional probabilities. However, since we deal with a Bernoulli random variable, the problem actually becomes a regression of the unknown function $\eta^*_j:\{0,1\}^{N-1}\mapsto [0,1]$, which can be performed using the previous normal dataset $\S_0$.

In this work we do not discuss the regression procedure, but we still need to note that non-parametric methods do seem suitable. Indeed, the estimation of the conditional distributions for every possible combination of fingerprints would lead to the estimation of $N(2^{N-1}-1)$ parameters. Note also that the Binary Tree or Random Forest regression algorithms seem well-adapted to this setting since the explanatory variables are binary. 
Let $\widehat{\eta_j}(\cdot):=\widehat{\eta_j}(\cdot ; X_1^0,\ldots,X_{n_0}^0)$ be our regressor. The first anomaly detection method one can think of is the simple `\textbf{\emph{plug-in}}' \textbf{method}: 
\begin{itemize}
\item fix $\delta$;
\item replace $\eta^*_j$ by $\widehat{\eta_j}$ in \Eq{eq:chernoff};
\item use \Eq{eq:chernoff} to obtain confidence levels for $M_t^{(j)}$.
\end{itemize}
\inlinetitle{Remark}{.} In practice, fixing $\delta$ is not trivial and simply taking a value below $0.05$ could lead to bad results. One way to fix $\delta$ is via cross-validation on the training stream. To do so, one should fix an acceptable false positive rate (\eg a standard value is $0.05$), then via cross-validation find the value of $\delta$ that generates a false positive rate lower than that fixed value.

However, \Eq{eq:hoeffding} is not true for the estimated version of $\eta_j^*$ and we must provide a concentration inequality around the estimated expectation $\widehat{\mu}=\sum_{i=1}^{n_t}\widehat{\eta}^{}_j(\x_i^{(-j)})$. In the following subsection, we give an asymptotic concentration inequality around our predicted number of shared events.

\subsection{Model-free prediction intervals}\label{subsec:thme}

\begin{theorem} \label{THME}
Let $\S_0$ be the training (normal) event stream for which we assume that $\forall i=1,\ldots,n_0$, $X^0_i \underset{\iid}{\sim} \Prob_{\!\!X^0}$. Let $\S_t$ be another stream for which the distribution $\Prob_{\!\!X}$ may be different but having the same support. 
Assume that both distributions have the same conditional probability function (\Definition{def:cpf}). 

\noindent
Assume our estimator $\widehat{\eta_j}$ is weakly consistent \textup{\cite{gyorfi2006distribution}}, 
and $\forall i=1,\ldots,n_0$,
\begin{equation}
\label{eq:bounded_diff}
\begin{split}
\underset{x_i,x'_i\in\{0,1\}^{N-1}}{\max}|\widehat{\eta}_j(&x;\ x_1,\ldots,x_i,\ldots,x_{n_0}) \\ 
 & \!\!\!\!\!-\widehat{\eta}_j(x;\ x_1,\ldots,x'_i,\ldots,x_{n_0})| \leq \kappa(n_0),
\end{split}
\end{equation}
where, $\kappa$ tends to $0$ when $n_0$ tends to infinity such that $n_0\kappa^2(n_0)\underset{n_0\rightarrow \infty}{\longrightarrow} 0$.
Then, we have :
\begin{equation}
\label{eq:concentration_min}
\begin{split}
\!\!\!\!\underset{n_0\rightarrow\infty}{\lim}\Prob\Bigg(\Bigg|&M^{(j)}_t - \sum_{i=1}^{n_t}\widehat{\eta}_j(X^{(-j)}_i;\ X_1^0,\ldots,X_{n_0}^0)\Bigg|>s\Bigg) \\
&\!\!\!\!\!\!\!\!\!\!\!\!\!\!\!\!\!\!\!\! \leq \underset{k\in[0,s]}{\min} \left\{2\exp\left(-\frac{2k^2}{n_t}\right)+2\exp\left(-\frac{(s-k)^2}{2n_t}\right)\right\}.
\end{split}
\end{equation}
\end{theorem}

\begin{proof} A sketch follows; the complete proof is provided in the Appendix. The successive use of the triangle, the Cauchy-Schwarz and the Jensen inequalities allows us to upper-bound 
$\Big|M^{(j)}_t - \underbrace{\sum_{i=1}^{n_t}\widehat{\eta}_j(X^{(-j)}_i;\ X_1^0,\ldots,X_{n_0}^0)}_{=\widehat{\mu}}\Big|$ by:
\begin{align*}
\begin{split}
\underbrace{\left|M_t^{(j)}-\mu^*\right|}_{(i)}&+\underbrace{\Big|\mu^*-\widehat{\mu}-\Exp\Big[\mu^*-\widehat{\mu}\Big]\Big|}_{(ii)} \\
&\ \ \, \ \, +n_t\underbrace{\Exp\Big[(\eta^*(X)-\widehat{\eta}(X))^2\Big]}_{(iii)}.
\end{split}
\end{align*}  
Since $(iii)$ tends to $0$ as $n_0$ tends to infinity, due to the consistency assumption, we can bound the left-hand side of inequality (\ref{eq:concentration_min}) by:
\begin{equation*}
\underset{k\in[0,s]}{\min}\Bigg\{\Prob\Big((i)>k\Big)+\underset{n_0\rightarrow\infty}{\lim}\Prob\Big((ii)>s-k\Big)\Bigg\}.
\end{equation*}
Applying Hoeffding's inequality on the first element of the sum, and McDiarmid's inequality on the second one, leads to the final result.

\end{proof}

\noindent
\textbf{Remark.} Replacing $k$ by $\frac{s}{3}$ in the final inequality given by \Theorem{THME}, we obtain:
\begin{equation}
\label{eq:concentration_pretty}
\begin{split}
\!\!\!\!\!\!\underset{n_0\rightarrow\infty}{\limsup}\Prob\Bigg(\Bigg|M^{(j)}_t - \sum_{i=1}^{n_t}\widehat{\eta}_j(X^{(-j)}_i&;\ X_1^0,\ldots,X_{n_0}^0)\Bigg|>s\Bigg) \\
&\ \ \, \ \, \leq 4\exp\left(-\frac{2s^2}{9n_t}\right).
\end{split}
\end{equation}

\inlinetitle{Remarks on \Theorem{THME}}{.} 
First of all, as mentioned in the first part of the theorem, the training and test event streams may follow different probability distributions. This is very interesting since, in practice, $M_t^{(j)}$ is a non-stationary time-series: \ie proportion of events in which a node is involved in is not stationary over time. However, we assume that, while in normal state, the probability that a node participates in an event, knowing the participation of the other nodes, does not change over time. From the network viewpoint, this means that the underlying graph structure, on which the events are dynamically created, does not change in that time as well.

Therefore, provided that all hypotheses are verified, we test whether the $\eta$ function has changed between the training and the test event streams; we test the \emph{stationarity of the conditional distributions}. Falling out of the confidence intervals (built with \Eq{eq:concentration_min} or \eqref{eq:concentration_pretty}) would indicate a significant change in the conditional probability. Consequently, a property of this method is that it enables the detection of changes in the activity level of a node, having as reference the activity of the other nodes in its close communication environment.

Besides real anomalies, one reason for our statistical test to see the observed communication volume to fall out of the confidence intervals is when the assumptions are not verified. The consistency may has not been reached yet, which means that the number of training samples is not large enough. The other reason may be that the support of the distribution has changed (\eg nodes sharing events for the first time), which is however important to be able to detect as well.

The consistency assumption is pretty typical for a regression framework. The reader may refer to the large literature that deals with this question \cite{gyorfi2002principles, gyorfi2006distribution} in which it has been shown that many regressors are consistent. As clarified earlier, this work does not aim to provide a new regression method, however, we must note that our method largely depends on the convergence rate of the estimator.

The last assumption we need to analyze is the bounded difference of \Eq{eq:bounded_diff}. In simple words, it says that when the size of the training set increases, a change of one sample does not affect much the estimated regression function. The second hypothesis, $n_0\,\kappa^2(n_0)\underset{n_0\rightarrow \infty}{\longrightarrow} 0$, is less intuitive. Nonetheless, for many estimators, $\kappa(n_0)=\mathcal{O}(\frac{1}{n_0})$ and thus the hypothesis holds. As an example, take the Nadaraya-Watson regressor \cite{gyorfi2006distribution}. Let here $K$ be the kernel function and $h$ the bandwidth. We then have $\kappa(n_0)=\frac{1/n_0}{\frac{1}{n_0}\sum_{i=1}^{n_0}K_h(X_i^{(-j)}-x)}=\mathcal{O}(\frac{1}{n_0})$, since $\frac{1}{n_0}\sum_{i=1}^{n_0}K_h(X_i^{(-j)}-x)$ converges to $\Exp[K_h(X^{(-j)}-x)]$.

\section{Experiments}\label{sec:experiments}

\subsection{State-of-the-art competitors}\label{sec:competitors}

For our comparative evaluation, we rely on the anomaly detection literature for dynamic graph (see \Sec{sec:related_works}). We choose state-of-the-art methods from the literature which, to the best of our knowledge, are the only existing works on the probabilistic anomaly detection at node-level, and hence are natural competitors to our work. They use the aggregated representation of the dynamic graph (see \Sec{sec:related_works} and \Fig{fig:g}). We set the aggregation's time-scale to one day, hence the edge weight between two nodes at a time interval corresponds to the number of events shared by those two nodes in that interval.

\inlinetitle{Heard's method}{} \cite{heard2010bayesian} consists in fitting, either sequentially (based on all past values) or retrospectively (based on all values but the one to predict), an homogeneous counting process on each edge of the graph independently. However, rather than focusing on edges, here we decided to model the total number of messages received per day by each node. We chose a retrospective fitting, as the number of studied timestamps is not large enough for efficient sequential fitting.

\inlinetitle{The Scan Statistics-based method}{} in \cite{wang2014locality}, at each timestamp, builds a statistic on the neighborhood around the node of interest, and normalizes it using past values in a time-window. The normalized statistic is used directly as an anomaly score. In our experiments, we used a statistic of order $0$, specifically, the weighted degree of the node of interest. With our aggregated graph construction, this corresponds to the sum of weights of the adjacent edges.

\inlinetitle{Anomaly scores}{.} We can build two anomaly scores. The first one, referred to as \emph{bilateral}, increases when the observed value is `far' from the expected one, in terms of absolute value. For our method, that simply corresponds to the score described in \ref{subsec:methodo}, \ie the \Eq{eq:anomaly_score} taken negatively so that it increases with the deviation from what is expected. 

The second anomaly score, referred to as \emph{unilateral}, is motivated by the fact that in telecommunication networks an interesting type of anomalous behavior is when a node has an abnormal low level of received messages. That may reflect an antenna breakdown. For this reason, the anomaly score should increase only when the observed value is lower than expected. For our method, we simply take $\rho = -\exp(\frac{-2\,(\widehat{\mu}-m_t^{(j)})}{ 9n_t})$.

\subsection{IoT dataset}\label{sec:sigfox_data}

The first results are obtained from a real industrial setting that concerns Sigfox, a telecommunication operator specialized on sensor and \emph{Internet-of-Things} (IoT) networks
\footnote{The datasets and our implementation of all compared anomaly detection methods are publicly available at \texttt{http://kalogeratos.com/psite/nad2019}.}. 
Networks like these are dedicated to cover objects or devices that need to exchange only little information with users avoiding standard transmission protocols (such as WiFi, 4G or Bluetooth) that may not be well-adapted to the operational constraints (\eg for low energy consumption). When a sensor needs to transmit a message, it simply sends a signal which can be received by several nearby Base Stations (BSs) that are in reachable distance. Our objective is to detect abnormal volume of received signals observed at any BS during a day. Hence, we consider that each sent message corresponds to a single event whose fingerprint spans only over the set of receiving BSs of the network. The value of each dimension of the event fingerprint indicates whether the message has been received or not by the corresponding BS ($1$ or $0$, respectively).

\begin{figure}[!t]
\includegraphics[scale=0.43, viewport=0 23 600 440, clip=true]{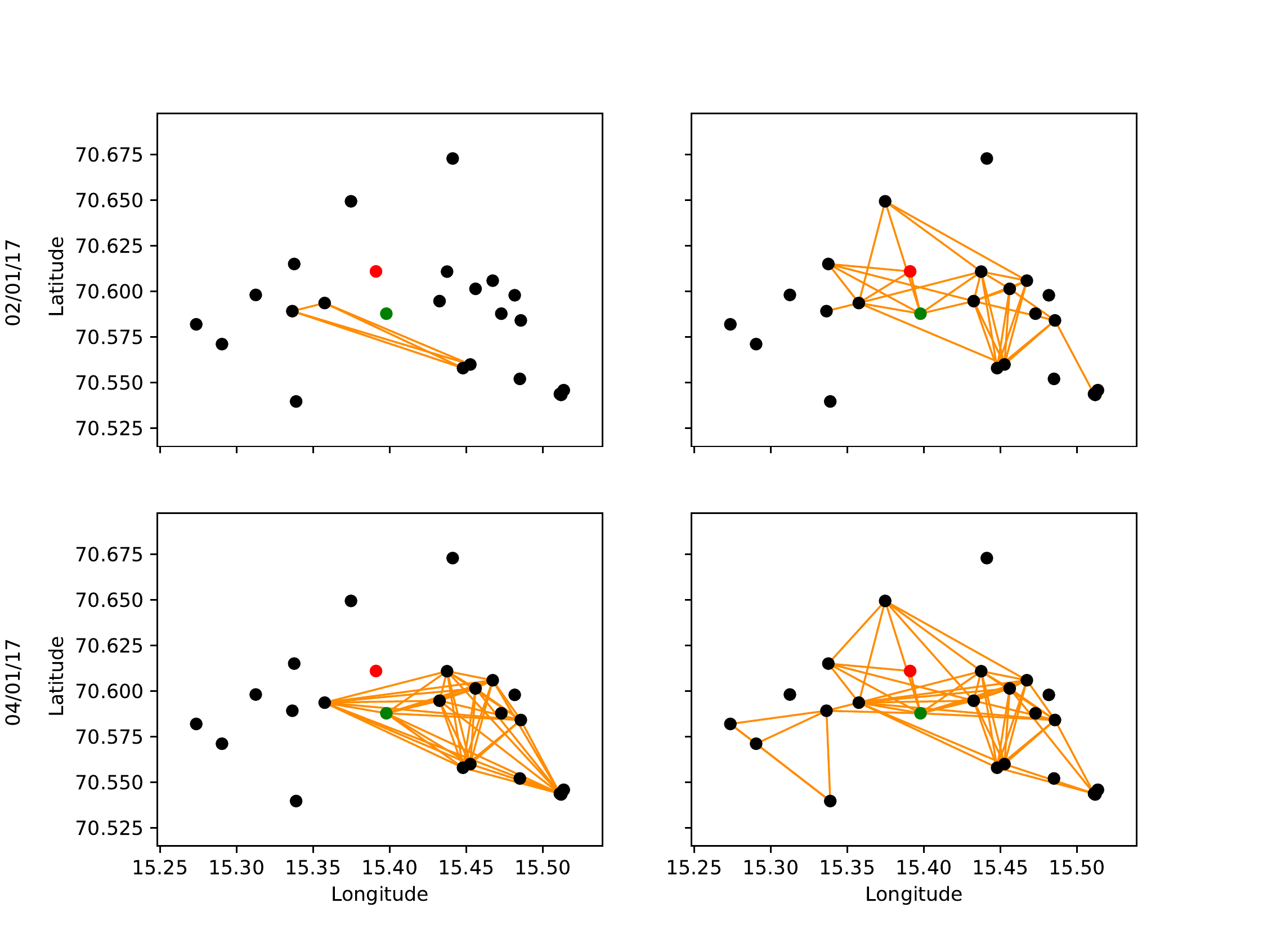}
\caption{High-resolution and aggregated representations (columns) of communication activity in a part of the considered Sigfox IoT network, during two consecutive days (rows). Each node corresponds to a Base Station (BS). The red BS has been tagged as \emph{anomalous} by experts and presents \emph{abnormal} behavior at some point during he observation time, while the green BS is taken as reference of \emph{normal} behavior. \textbf{Left column:}~Each graph represents a single event that occurred at the day indicated on the left. The involved nodes form a clique in the network. \textbf{Right column:}~Each graph is an aggregated representation of all the events of the respective day. A link is drawn when two nodes share more than 30\% of the total number of messages they received during that day.}
\label{fig:g}
\end{figure}

In this evaluation study we use the event stream recorded at a subset of $34$ BSs over a period of $5$ months. \Fig{fig:g} shows the relative geographical locations of the BSs. Each BS is a node in the considered graph representation and each event creates instantaneously a clique in the graph (\eg see the left column of \Fig{fig:g}) among the involved nodes that all receive the same message, sent from the same device). 

The results of \Fig{fig:res} concerns two BSs: one with a known anomaly (lying between the two vertical red lines of \Fig{fig:res}a), the other with no known issues (\Fig{fig:res}b). Note that we have the opinion of Sigfox's experts only about these two BSs, yet we lack labels for the rest of the BSs. According to the experts, network's operation has been normal during January 2017, thus, for both reference BSs the learning phase was performed during that period. We used a Random Forest regressor \cite{breiman2001random} as implemented in \cite{scikit-learn}. The testing phase was performed independently on a daily basis for the subsequent $4$ months. In other words, and this concerns all our results, we report the raw outcome of the independent daily detection for anomalies without applying any post-processing that could certainly improve the performance of most methods. \Fig{fig:res} refers to the testing phase and shows the evolution of the observed number of received messages (blue) and the evolution of the confidence levels (orange) with $\delta = 0.01$. 

\begin{figure}[!t]
\subfigure[Evolution of confidence levels for the anomalous BS.]{\label{fig:a}%
	\centerline{\includegraphics[width=1\linewidth, viewport=20 32 550 350, clip=true]{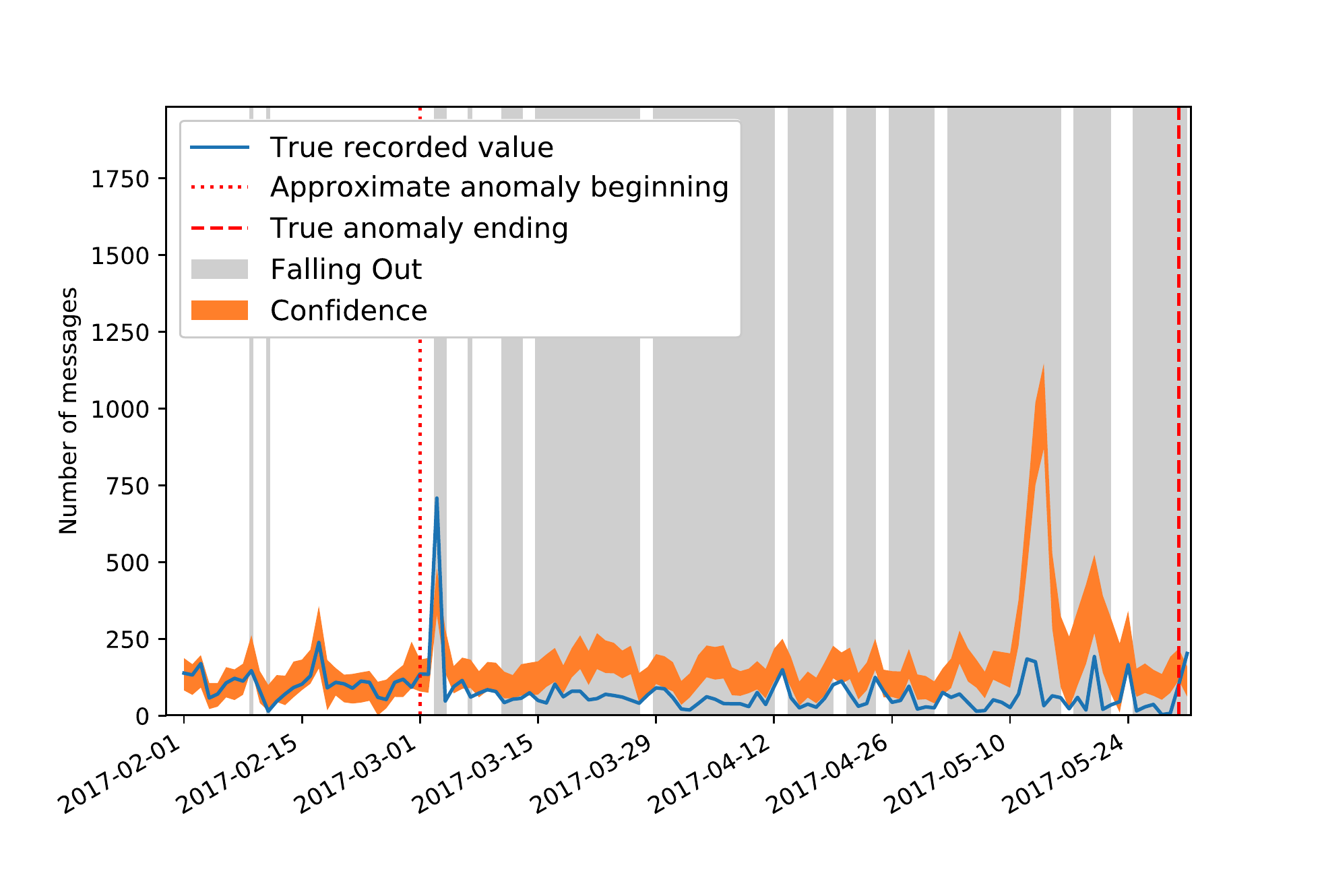}}%
}
\subfigure[Evolution of confidence levels for the normal BS.]{\label{fig:b}%
	\centerline{\includegraphics[width=1\linewidth, viewport=20 32 550 350, clip=true]{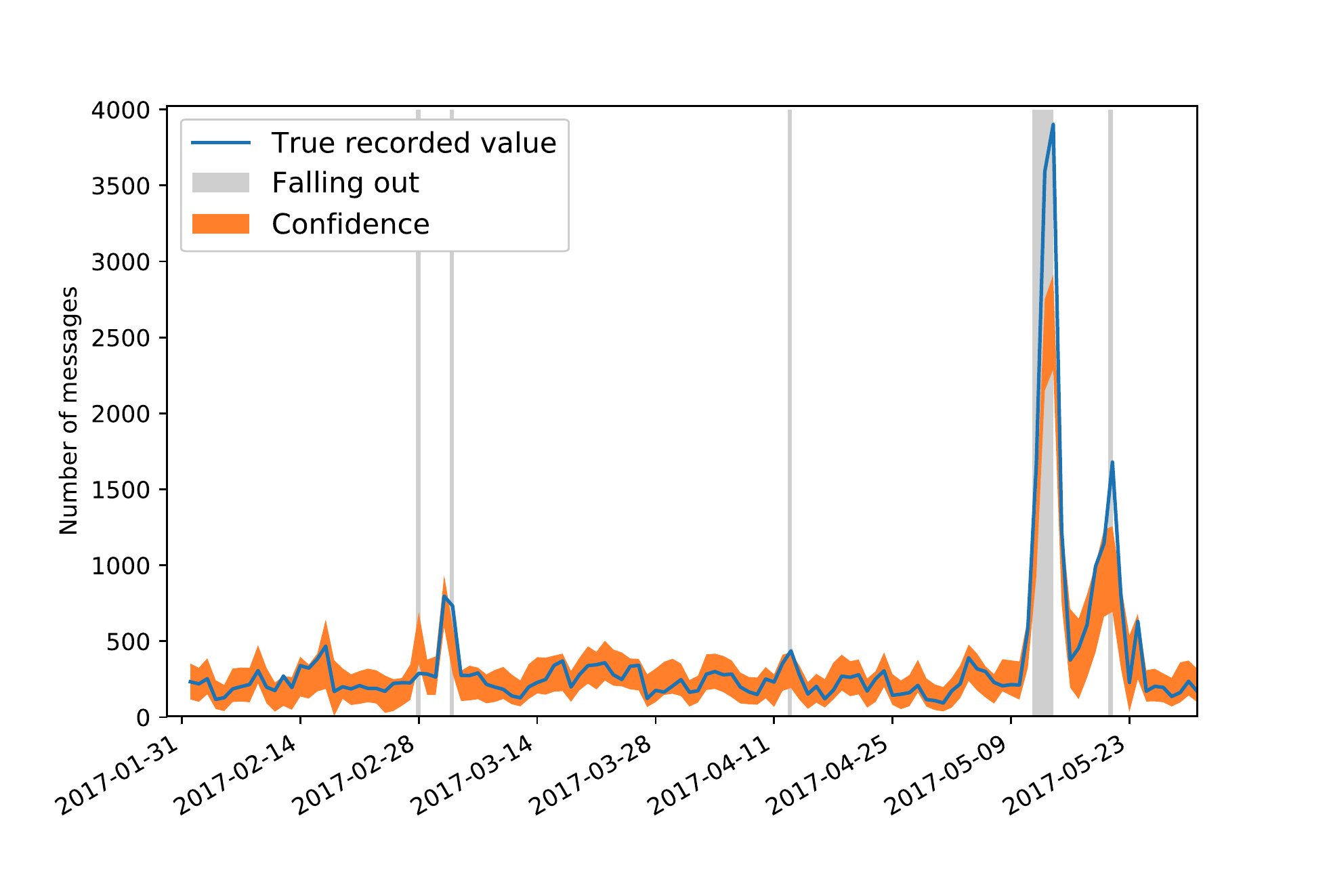}}%
}
\subfigure[ROC curves for bilateral and unilateral confidence levels. Comparison with stationary counting processes.]{\label{fig:c}%
  \centerline{\includegraphics[width=0.8\linewidth, viewport=20 10 435 310, clip=true]{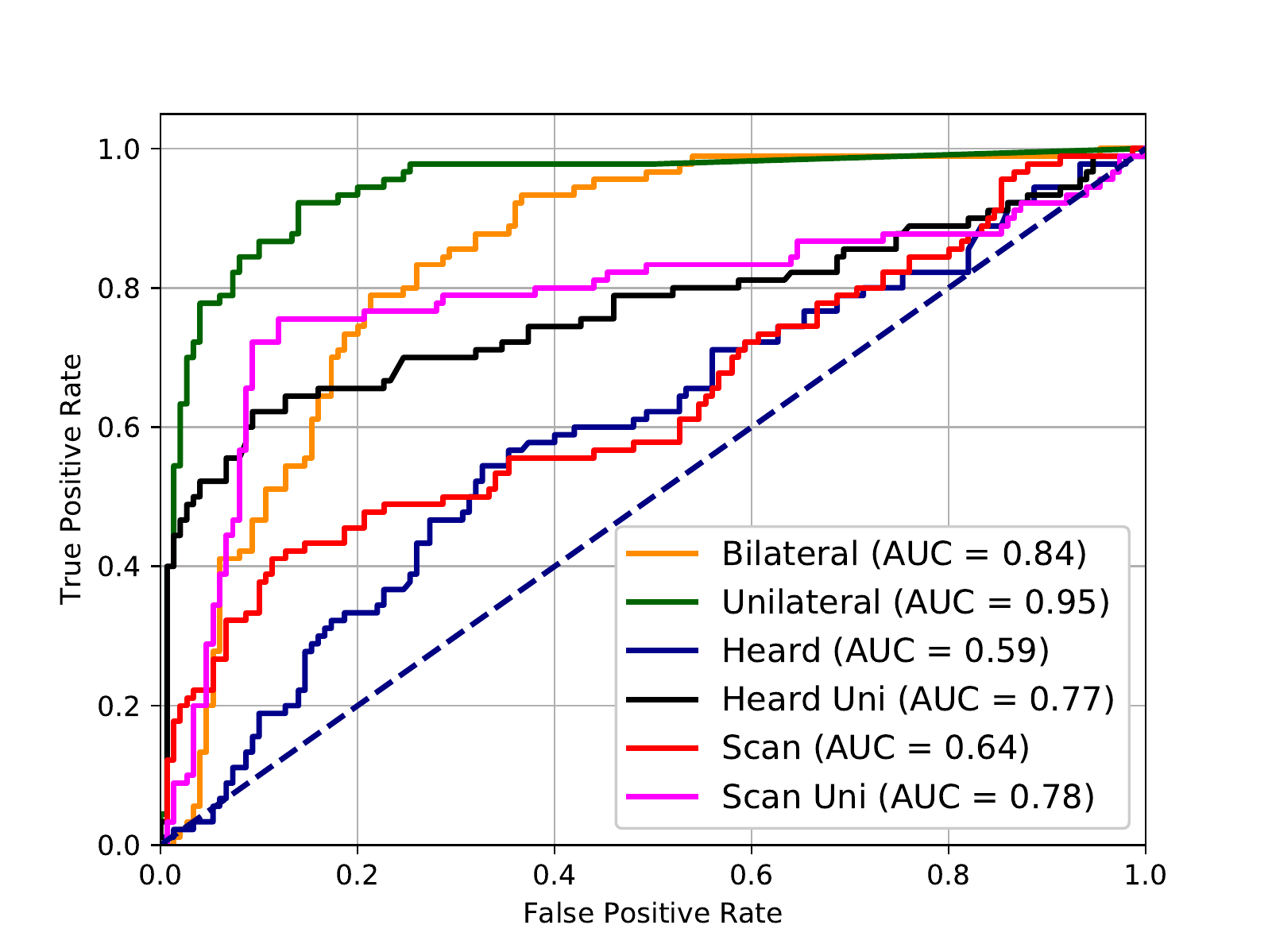}}%
}
\caption{Results on two BSs of the considered IoT network. \textbf{(a{--}b)}:~The true number of  messages received by the abnormal BS and the normal one over the testing period. The yellow area corresponds to the predicted confidence region for the number of received messages. \textbf{(c)}:~The ROC curves and their AUC of the proposed method using a bilateral (orange) and an unilateral (green) anomaly score. Comparison with Heard's \cite{heard2010bayesian} and Scan Statistics \cite{wang2014locality, priebe2005scan}.}
\label{fig:res}
\end{figure}

The results, especially the ROC curves, show that our method (bilateral and unilateral variations) outperforms the compared approaches. As expected, the tests with unilateral score were always better than those with bilateral, for all the detection methods. \Fig{fig:res}b suggests that our model is well-suited for the analysis of the BSs in normal network operation. Indeed, the false positive rates are pretty low in that case. 

To prove this latter idea, we applied our method on $5$ other BSs which are located close to each other. The predicted confidence region around the predicted value are plotted for each BS in \Fig{fig:other}. Once again, we can see that the observed number of received messages falls out of the confidence level very few times. The fact that our method reports long anomalies for many BSs during May 11\,-\,25, may be a sign that retraining is needed. However, for the third BS, the observed value is persistently very low compared to the predicted confidence intervals, which is a stronger indication for anomalous behavior during that period of time.

\begin{figure}[!t]
\includegraphics[scale=0.589, viewport=0 50 1000 490, clip=true]{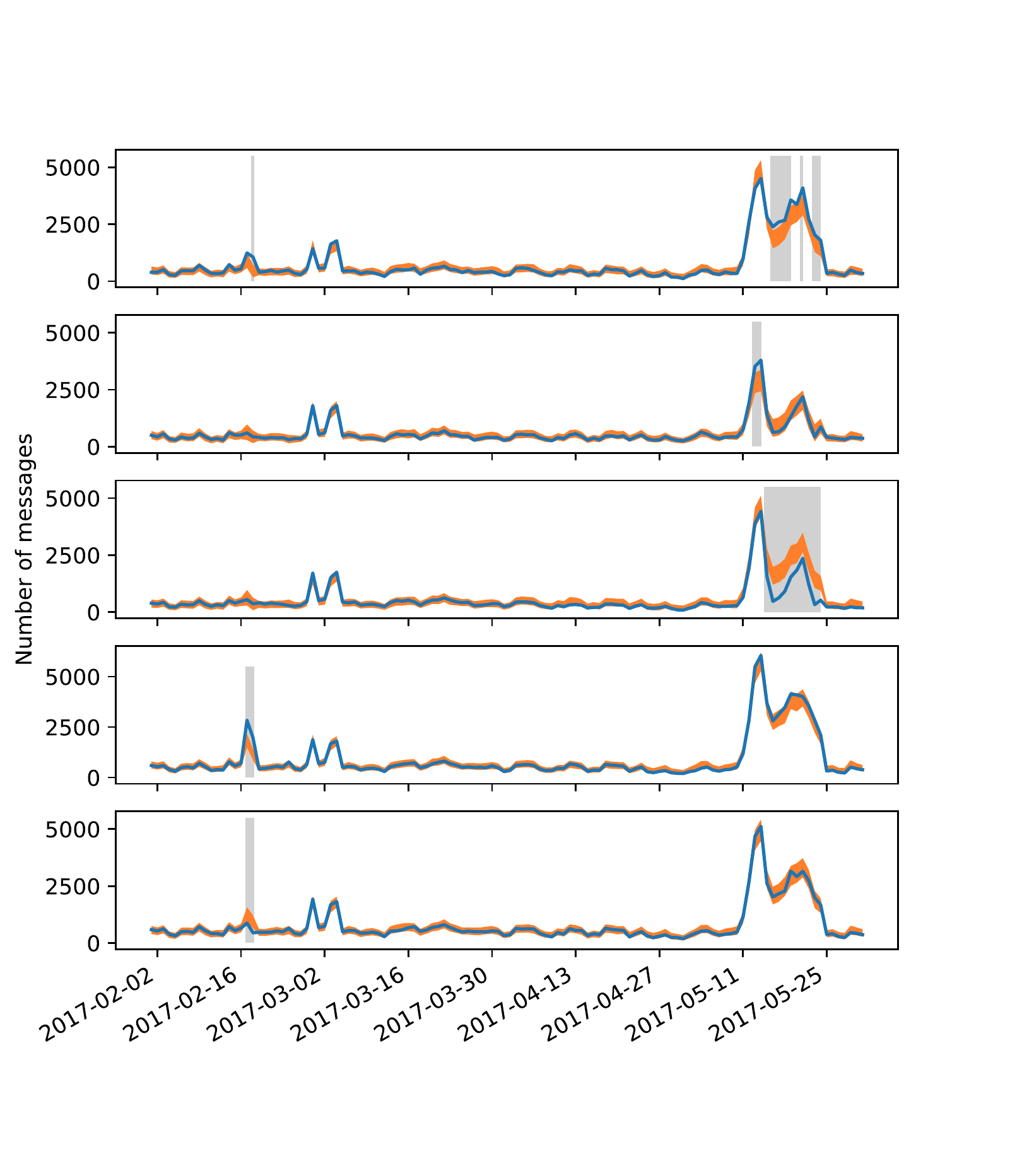}%
\vspace{-4mm}
\caption{Evolution of the predicted confidence regions for five 
BSs.}\label{fig:other}
\end{figure}

\subsection{Simulated dataset}\label{sec:simulation}
Aiming to extend the scale of our experimental study, we developed a data generator that simulates network communication activity. To be consistent with the previous experiment of \Sec{sec:sigfox_data}, we keep the nodes' spatial arrangement of Sigfox network. We propose the following simulation process:
\begin{itemize}
\item[S1)]\emph{Sample the spatial network structure:} Draw $N$ node (\ie analogous to BSs) locations, according to a mixture model $\mathcal{M}$ of $K$ (bivariate) Gaussian distributions.

\item[S2)]\emph{Sample an event/fingerprint:} First generate a transmission location (analogous to a device) $\ell\sim\mathcal{M}$, as in Step~1. Then for each node, let its location $x$, draw a Bernoulli with a parameter inversely proportional to the distance $d(x,\ell)$. In our experiments, we set the Bernoulli parameters to be equal to $\exp(-\frac{1}{\sigma_x}d(x,\ell))$, where $\sigma_x$ is a \emph{location-dependent visibility} parameter that controls the density of the graph.

\item[S3)]\emph{Generate an event stream (clique stream):} At each timestamp $t$, draw $n_t$ event fingerprints by applying Steps~2-3, where $n_t$ may be constant, or random, over time.

\item[S4)]\emph{Simulate anomalies through non-stationarity:} The simplest way to simulate non-stationarity is to draw the total number of events at each timestamp according to a non-stationary process. To increase the complexity of the phenomena, one may also let the component (or cluster) proportion of the mixture in $\mathcal{M}$  to vary at each timestamp. That would correspond to the case where devices appear following a non-uniform spatial distribution. To simulate anomalies for a node, it is sufficient to let vary the visibility parameter $\sigma_x$ associated to the node's location.
\end{itemize}

In order to demonstrate the robustness of our method, we apply the above generative process in three simulations with different `complexity', whereas sharing the following properties: 
\begin{itemize}
\item S1: $N=100$ communication nodes are drawn, for which, $T=1100$ timestamps are then simulated. The number of Gaussian distributions are fixed to $K=10$.
\item S2: The same set of constant visibility parameters is used.
\item S3: The first $500$ timestamps are treated as the training stream, while the rest correspond to the test stream.
\item S4: A single arbitrary node is chosen to be anomalous. For which, $4$ anomaly time intervals are simulated: $[750,\,800]$, $[850,\,900]$, $[950,\,1000]$ and $[1050,\,1100]$. Each of these intervals imitates an anomalous behavior at a different scale; this is achieved by decreasing only the visibility parameter $\sigma_x$ associated with the anomalous node.
\end{itemize}

\begin{figure*}[t]
\subfigure[Activity of the anomalous node during \Experiment{1}.] {\label{fig:easy_a}%
	\includegraphics[width=0.325\linewidth, viewport=20 10 435 310, clip=true]{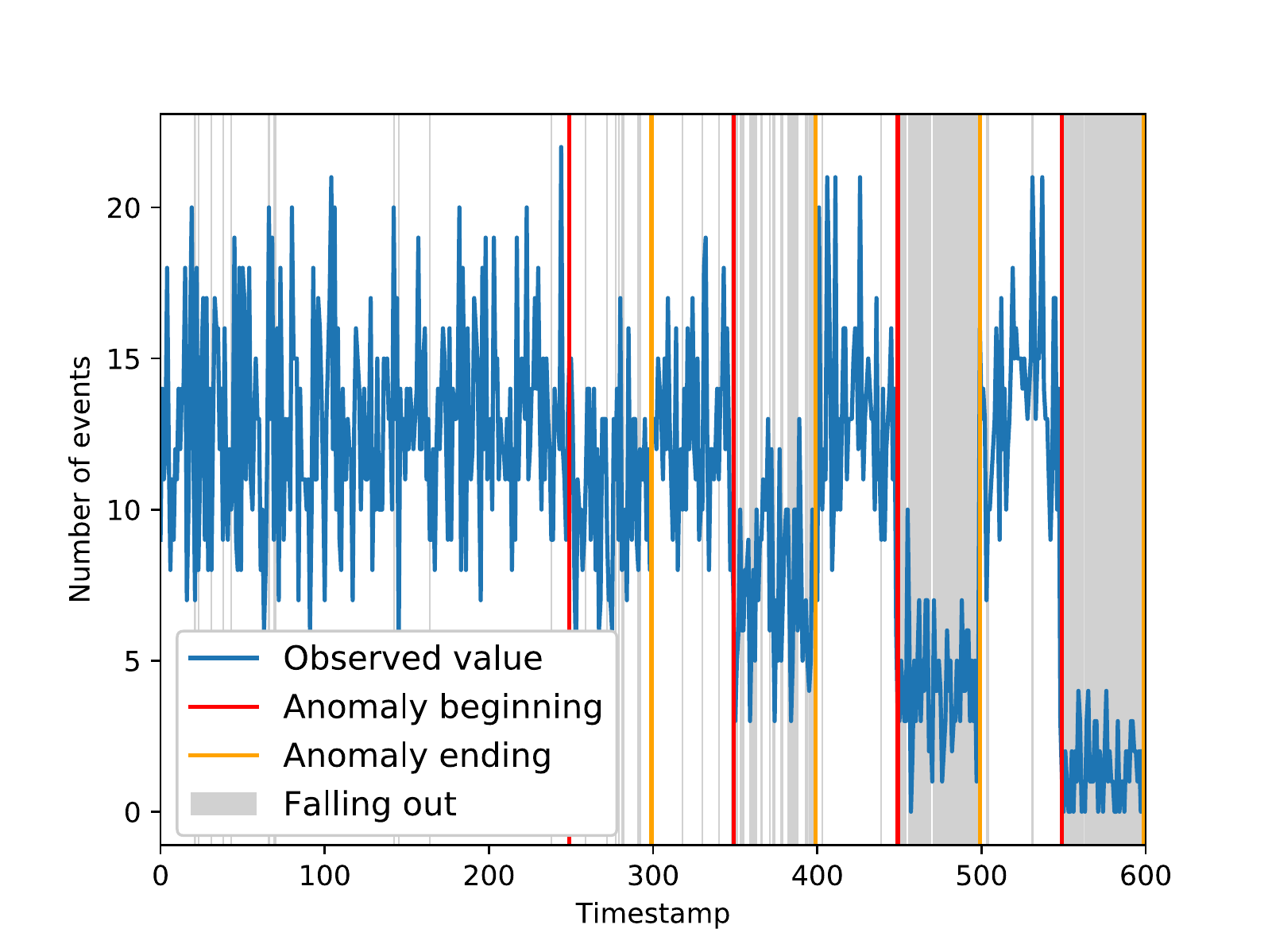}%
}%
\hspace{0.5em}
\subfigure[Activity of the anomalous node during \Experiment{2}.]{\label{fig:middle_a}%
	\includegraphics[width=0.325\linewidth, viewport=20 10 435 310, clip=true]{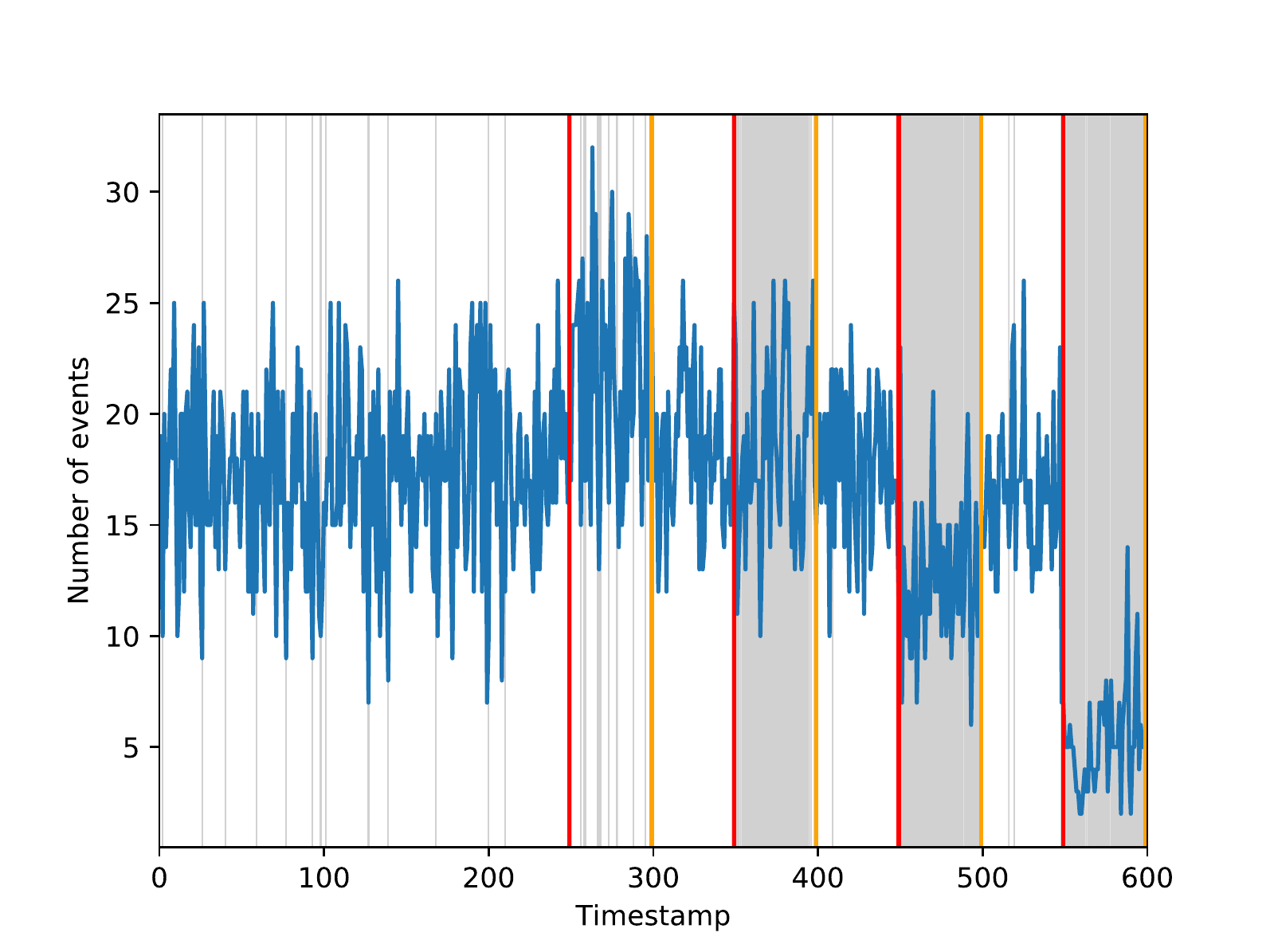}%
}%
\hspace{0.5em}
\subfigure[Activity of the anomalous node during \Experiment{3}.]{\label{fig:hard_a}%
	\includegraphics[width=0.325\linewidth, viewport=20 10 435 310, clip=true]{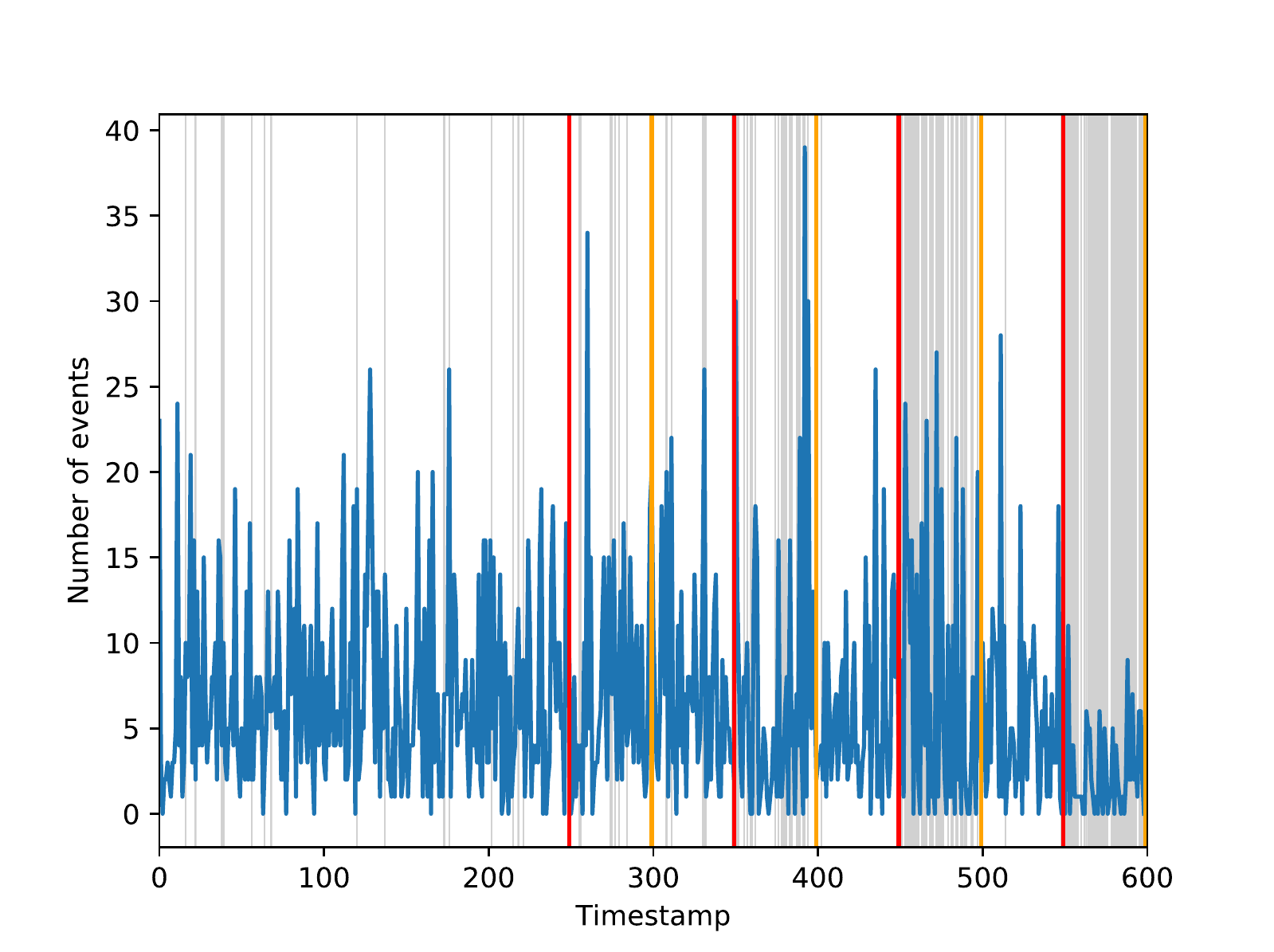}%
}
%
\vspace{0.1cm}
\subfigure[ROC curves for \Experiment{1}.]{\label{fig:easy_b}%
   \includegraphics[width=0.325\linewidth, viewport=20 10 435 310, clip=true]{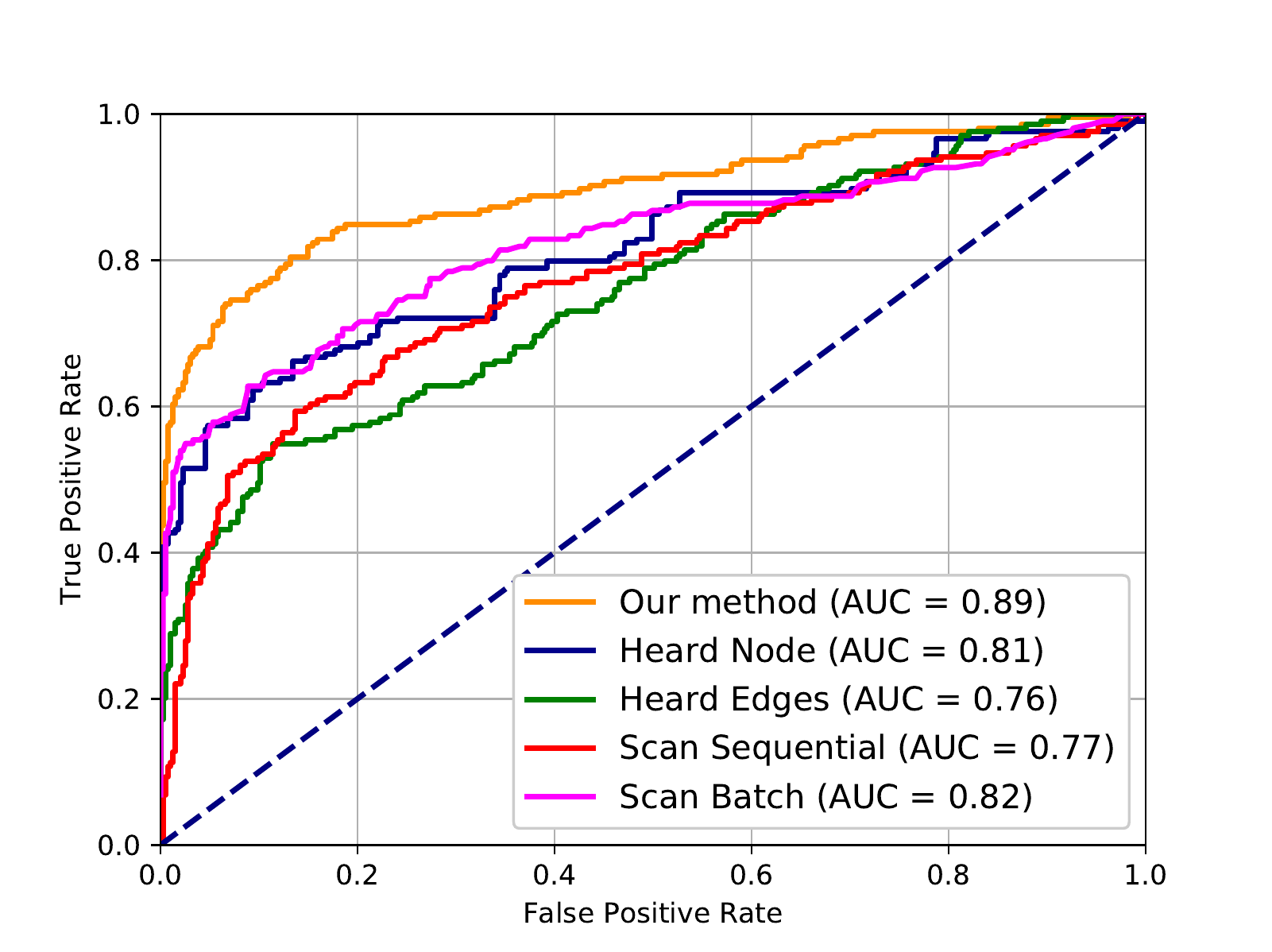}%
}%
\hspace{0.5em}
\subfigure[ROC curves for \Experiment{2}.]{\label{fig:middle_b}%
  \includegraphics[width=0.325\linewidth, viewport=20 10 435 310, clip=true]{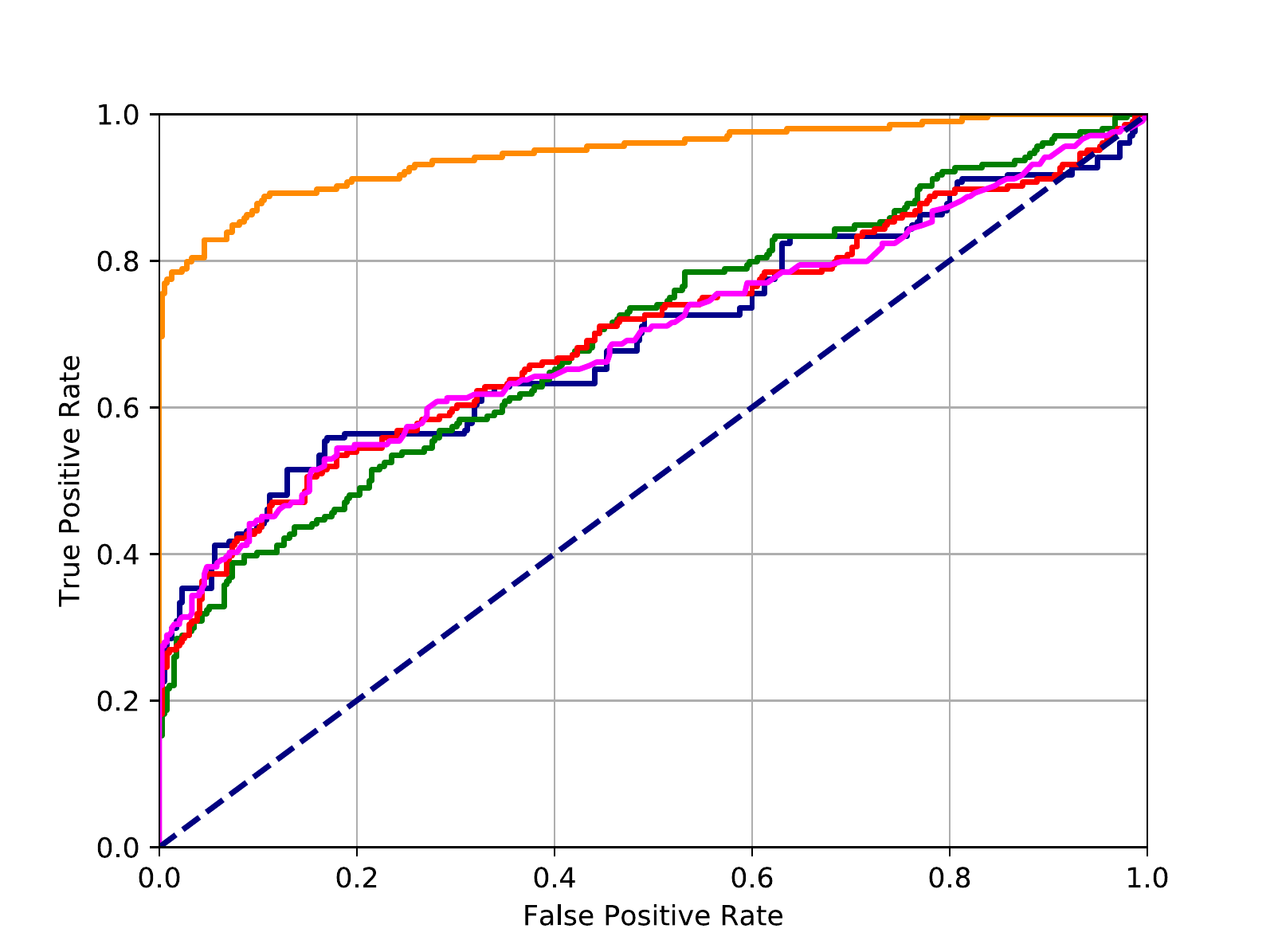}%
}%
\hspace{0.5em}
\subfigure[ROC curves for \Experiment{3}.]{\label{fig:hard_b}%
	\includegraphics[width=0.325\linewidth, viewport=20 10 435 310, clip=true]{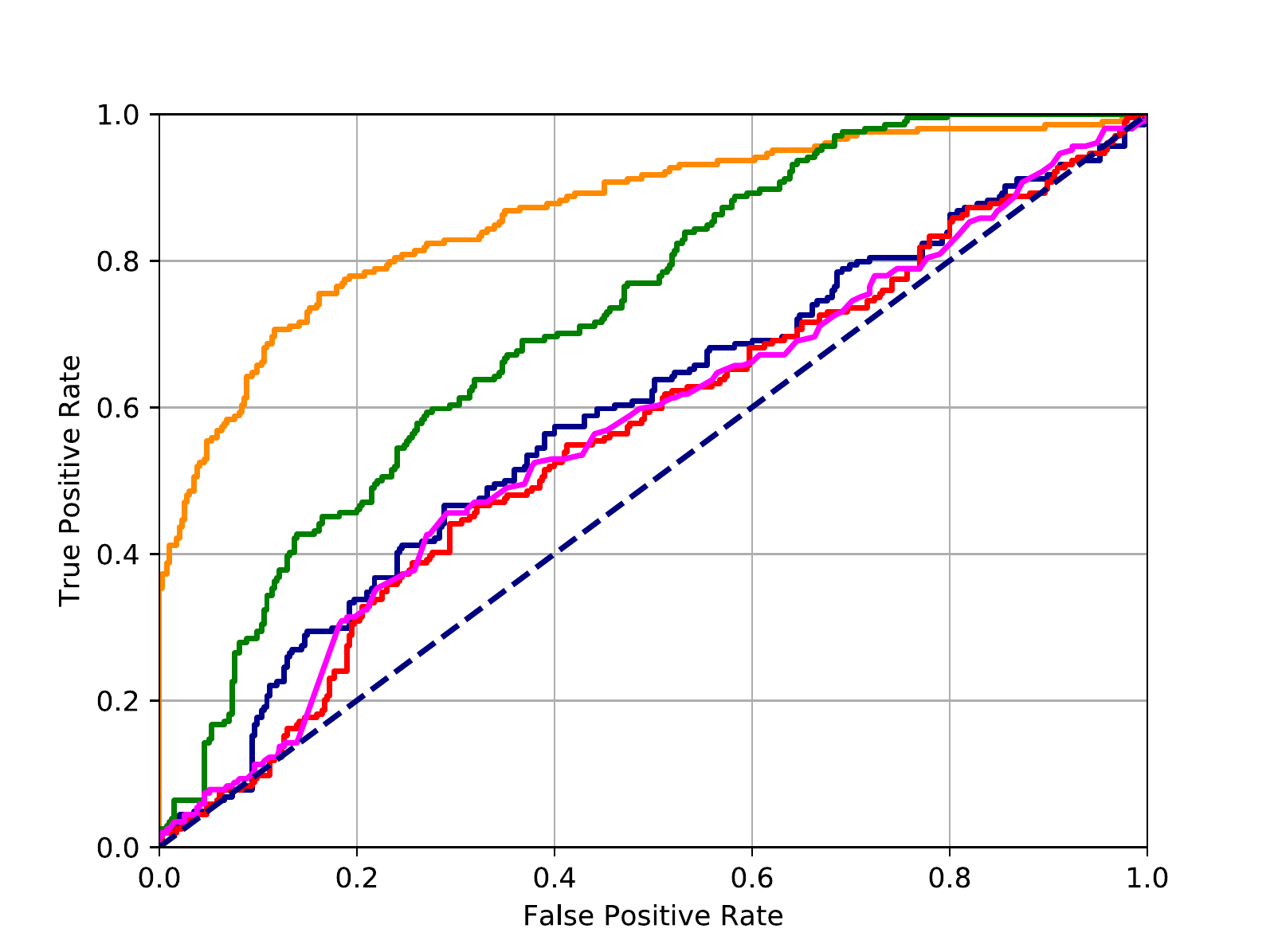}%
}
\caption{Results on simulated communication streams. The columns report results related to the three generated streams in order, \Experiment{1,\,2,\,3}, respectively. \textbf{(a-c)} The time-series of the number of messages (\ie communication events) received by the anomalous node during the testing period of each experiment. Four anomaly intervals are simulated for the same fixed node that is chosen to act as anomalous, in the time intervals [750,\,800], [850,\,900], [950,\,1000], and [1050,\,1100]. The beginning and the end timestamps of each anomalous interval are indicated with red and orange vertical lines, respectively, in the plots. 
\textbf{(d-f)} the ROC curves of the node-level outlier detection task for the three synthetic streams.}
\label{fig:sim_results}
\end{figure*}

Our three experiments (\Experiment{1}-3) differ in their complexity regarding the stationarity of the respective time-series, \ie number of events in which the \emph{anomalous node} (the node that at some point develops an anomalous behavior) participates in each experiment. The top row of \Fig{fig:sim_results} presents the time-series of the test streams. The timestamps of the beginning and the end of each simulated anomalous behavior are also indicated in the plots with orange and red vertical lines, respectively. 

In \Experiment{1} (\Fig{fig:sim_results}a,\,d), the process is perfectly stationary: at each timestamp, exactly $100$ events are generated with the same process. In \Experiment{2} (\Fig{fig:sim_results}b,\,e), the total number of events participated at each timestamp remains $100$, with the difference that a Dirichlet random variable of order $K$ is drawn, with parameters all equal to $1$. This corresponds to the mixing variable (\ie proportion) for the components of $\mathcal{M}$. The last one, \Experiment{3} (\Fig{fig:sim_results}c,\,f), is meant to be more difficult: it uses the Dirichlet mixing as well, however, at each timestamp of anomaly, the total number of generated events is increased. This is a `tricky' setting for the bare human eye as the time-series of interest `looks' stationary (\Fig{fig:sim_results}c) although there are actual changepoints in the node behavior.

For all three experiments, the threshold value that needs to be fixed for building the confidence levels was estimated using cross-validation (see details in \Sec{subsec:methodo}). We fix the acceptable false positive rate at $0.05$. The light gray vertical lines in the background of the top row plots in \Fig{fig:sim_results} indicate the timestamps at which the observed values fall out of the confidence levels, and as such they can be spotted as outliers. 

The bottom row of \Fig{fig:sim_results} shows the ROC curves of the node-level outlier detection task for the three synthetic streams. The competitors are the same as those of the experiment on real data (\Fig{fig:res}), but here only bilateral scores are plotted. In addition, here we employ a second version of both Heard's and Scan Statistics methods. The \emph{Heard Edges} fits an homogeneous Poisson process on each edge independently; each edge denoting the number of common messages received by two nodes. In this case, the node anomaly score is the sum of $p$-values of its edges. Moreover, the \emph{Scan Batch} method simply outputs anomaly scores equal to the normalized deviation of the statistic of interest (see \Sec{sec:competitors}) from a (unordered) batch of the training stream, hence without sequential analysis.

All the reported results indicate that the proposed method outperforms clearly its competitors. As expected by its design, our approach is shown to be robust to the non-stationarity introduced at arbitrary timestamps during our simulations. The performance of all other methods seems to decrease fast with the increase of non-stationarity (\ie behavior complexity). An important closing remark is to remind that, in our evaluation, we have been applying anomaly detection independently on each day. As this work is related to the detection of changepoints in nodes' behavior rather than instantaneous anomalies, post-processing (such as filtering) of the raw detection outcome could increase the accuracy of most methods.

\section{Conclusions}\label{sec:conclusions}

In this paper we presented a probabilistic framework for node-level anomaly detection in communication networks. We went beyond the aggregated representations that the existing literature has used to model the communication activity. Instead, we modeled such activity as a \emph{clique stream} where each event creates an instantaneous clique among the communicating nodes of the graph. The detection approach we proposed is to infer the conditional probabilities of cliques to be generated. This allowed the derivation of node anomaly scores which are efficient in detecting when the communication volume deviates from the `normal' behavior (estimated using a training stream of normal communication behavior), while also being statistically interpretable. We applied our method on both real-world and synthetic sensor network data, and demonstrated that it outperforms other probabilistic approaches found in the related literature.

As future work, there is room to further improve the accuracy of the statistical modeling, consider that events can create more complex structures of connected nodes than cliques, include dynamics coming from (dis)appearance of nodes, and finally bring our method closer to the link prediction or structure inference tasks, using for instance the learned conditional probabilities. 

\section*{Acknowledgments}
Part of this work was funded by the IdAML Chair hosted at ENS-Paris-Saclay.

\balance


\newpage
\section*{Appendix}
\noindent\textbf{\Theorem{THME}.} \textit{See \Sec{subsec:thme}}.
\begin{proof}
In the following, we note $\widehat{\eta}_{n_0}(\cdot):=\widehat{\eta_j}(\,\cdot\,;\ X_1^0,\ldots,X^0_{n_0})$. We also assume the distribution described in \Theorem{THME}: $\forall i=1,\ldots,n_0$, $X_i\underset{\iid}{\sim} \Prob_{\!\!X^0}$ and $\forall j=1,\ldots,n_t$, $X_j\underset{\iid}{\sim} \Prob_{\!\!X}$. With an abuse of notation, $\Prob_{\!\!X^0}$ and $\Prob_{\!\!X}$ also refer to the marginal distributions. 
Using the triangle inequality, we get:
\begin{align*}
&\left|M^{(j)}_t- \sum_{i=1}^{n_t}\widehat{\eta}_{n_0}(X^{(-j)}_i)\right| \\ 
&\leq \left|M^{(j)}_t - \sum_{i=1}^{n_t}\eta^*(X^{(-j)}_i)\right|+\left|\sum_{i=1}^{n_t}(\eta^*(X^{(-j)}_i)-\widehat{\eta}_{n_0}(X^{(-j)}_i))\right| \\
&\leq \left|M^{(j)}_t - \sum_{i=1}^{n_t}\eta^*(X^{(-j)}_i)\right| \\
& \quad + \Bigg|\sum_{i=1}^{n_t}(\eta^*(X^{(-j)}_i)-\widehat{\eta}_{n_0}(X^{(-j)}_i))- \\
    & \quad \quad \quad\quad \quad\Exp_{\!X^0\otimes X}\left[\sum_{i=1}^{n_t}(\eta^*(X^{(-j)}_i)-\widehat{\eta}_{n_0}(X^{(-j)}_i))\right]\Bigg| \\ 
& \quad + \underbrace{\left|\Exp_{\!X^0\otimes X}\left[\sum_{i=1}^{n_t}(\eta^*(X^{(-j)}_i)-\widehat{\eta}_{n_0}(X^{(-j)}_i))\right]\right|}_{(*)}.
\end{align*}
In the above, $\Exp_{\!X^0\otimes X}$ means that the expectation is taken with distribution $\Prob_X^0$ for $\S_0$ and $\Prob_{\!\!X}$ for $\S_t$.
Using Jensen's inequality and the fact that all $X_i^{(-j)}$ are \iid, we get:
\begin{align*}
\raisebox{1mm}{$_{(*)}$} &\leq \Exp_{\!X^0\otimes X}\left[\left|\sum_{i=1}^{n_t}\eta^*(X^{(-j)}_i)-\widehat{\eta}_{n_0}(X^{(-j)}_i)\right|\right] \\
&\leq n_t \Exp_{\!X^0\otimes X}\left[\left|\eta^*(X)-\widehat{\eta}_{n_0}(X)\right|\right] \\
 & \leq n_t\Exp_{\!X^0}\left[\int \left|\eta^*(\x)-\widehat{\eta}_{n_0}(\x)\right|\Prob_{\!\!X}(\text{d}\x)\right]. \\
 & \text{Using Cauchy-Schwarz inequality:} \\
 & \leq n_t\Exp_{\!X^0}\Bigg[\sqrt{\int \left(\eta^*(\x)-\widehat{\eta}_{n_0}(\x)\right)^2\Prob_{\!\!X^0}(\text{d}\x)} \times \\ 
 & \quad \quad \quad \quad \quad \quad \times \sqrt{\int\frac{\Prob_{\!\!X}(\x)}{\Prob_{\!\!X^0}(\x)}\Prob_{\!\!X^0}(\text{d}\x)}\Bigg]. \\
 & \text{With the same support hypothesis:} \\
 & = n_t\Exp_{\!X^0}\left[\sqrt{\int \left(\eta^*(\x)-\widehat{\eta}_{n_0}(\x)\right)^2\Prob_{\!\!X^0}(\text{d}\x)}\right]. \\
 & \text{Using Jensen inequality:} \\
 & \leq \sqrt{\Exp_{\!X^0\otimes X^0}\left[(\eta^*(X)-\widehat{\eta}_{n_0}(X))^2\right]} := r(n_0).
\end{align*}
Due to the assumption of weak consistency, $r(n_0)$ converges to zero, so as $(*)$.  
In the following, we assume that $r(n_0)<s$ which is always true after a certain rank. We note $\tilde{s}(n_0)=s-r(n_0)$.
Back to the first inequality of the proof, we get $\forall k\in\left(0,\tilde{s}(n_0)\right)$:
\begin{align*}
&\Prob(|M^{(j)}_t - \sum_{i=1}^{n_t}\widehat{\eta}_{n_0}(X^{(-j)}_i)|>s)\\
&\leq \Prob(|M^{(j)}_t - \sum_{i=1}^{n_t}\eta^*(X^{(-j)}_i)|>k) \\
&  + \Prob\Bigg(\Bigg|\sum_{i=1}^{n_t}(\eta^*(X^{(-j)}_i)-\widehat{\eta}_{n_0}(X^{(-j)}_i))\\
    & \quad - \Exp_{\!X^0\otimes X}\left[\sum_{i=1}^{n_t}(\eta^*(X^{(-j)}_i)-\widehat{\eta}_{n_0}(X^{(-j)}_i))\right]\Bigg|>\tilde{s}(n_0)-k\Bigg).
\end{align*}
This is due to the fact that $k+\tilde{s}(n_0)-k+r(n_0)=s$. 
We now need to find an upper bound on the two elements of the right-hand side of the previous inequality.
The first element of the sum is easily bounded using Jensen's inequality:
\begin{equation*}
\Prob\left(|M^{(j)}_t - \sum_{i=1}^{n_t}\eta^*(X^{(-j)}_i)|>k\right)\leq 2\,\exp\left(-\frac{2k^2}{n_t}\right)\!.
\end{equation*}
For the second element, we must note that $\sum_{i=1}^{n_t}(\eta^*(X^{(-j)}_i)-\widehat{\eta}_{n_0}(X^{(-j)}_i))$ is a function, with bounded differences, of $n_0+n_t$ independent random variables. Thus, we can apply McDiarmid's inequality to bound our probability:
\begin{align*}
\Prob&\Bigg(\Bigg|\sum_{i=1}^{n_t}(\eta^*(X^{(-j)}_i)-\widehat{\eta}_{n_0}(X^{(-j)}_i)) \\
& \quad  - \Exp_{\!X^0\otimes X}\left[\sum_{i=1}^{n_t}(\eta^*(X^{(-j)}_i)-\widehat{\eta}_{n_0}(X^{(-j)}_i))\right]\Bigg|>\tilde{s}(n_0)-k\Bigg) \\ 
& \leq 2\exp\left(-2\frac{\left(\tilde{s}(n_0)-k\right)^2}{4n_t+n_t^2n_0\kappa^2(n_0)}\right).
\end{align*}

\noindent This implies that:
\begin{align*}
\Prob(|M^{(j)}_t - &\sum_{i=1}^{n_t}\widehat{\eta}_{n_0}(X^{(-j)}_i)|>s) \\
&\leq 2\exp\left(-\frac{2k^2}{n_t}\right)+2\exp\left(-\frac{2\left(\tilde{s}(n_0)-k\right)^2}{4n_t+n_t^2n_0\kappa^2(n_0)}\right).
\end{align*}
This is true $\forall k\in\left(0,\tilde{s}(n_0)\right)$. Furthermore, since $\tilde{s}(n_0)\underset{n_0\rightarrow \infty}{\longrightarrow} s$ and $n_0\kappa^2(n_0)\underset{n_0\rightarrow \infty}{\longrightarrow} 0$, passing to the limit on both side of the previous equation, we get our final result:
\begin{equation*}
\begin{split}
\underset{n_0\rightarrow\infty}{\lim}\Prob(|&M^{(j)}_t - \sum_{i=1}^{n_t}\widehat{\eta}_{n_0}(X^{(-j)}_i)|>s) \\
& \leq \underset{k\in[0,s]}{\min} \left\{2\exp\left(-\frac{2k^2}{n_t}\right)+2\exp\left(-\frac{(s-k)^2}{2n_t}\right)\right\}.
\end{split}
\end{equation*}
\end{proof}


\begin{thebibliography}{10}
\providecommand{\url}[1]{#1}
\csname url@samestyle\endcsname
\providecommand{\newblock}{\relax}
\providecommand{\bibinfo}[2]{#2}
\providecommand{\BIBentrySTDinterwordspacing}{\spaceskip=0pt\relax}
\providecommand{\BIBentryALTinterwordstretchfactor}{4}
\providecommand{\BIBentryALTinterwordspacing}{\spaceskip=\fontdimen2\font plus
\BIBentryALTinterwordstretchfactor\fontdimen3\font minus
  \fontdimen4\font\relax}
\providecommand{\BIBforeignlanguage}[2]{{%
\expandafter\ifx\csname l@#1\endcsname\relax
\typeout{** WARNING: IEEEtran.bst: No hyphenation pattern has been}%
\typeout{** loaded for the language `#1'. Using the pattern for}%
\typeout{** the default language instead.}%
\else
\language=\csname l@#1\endcsname
\fi
#2}}
\providecommand{\BIBdecl}{\relax}
\BIBdecl

\bibitem{latapy2017stream}
M.~Latapy, T.~Viard, and C.~Magnien, ``Stream graphs and link streams for the
  modeling of interactions over time,'' \emph{preprint arXiv:1710.04073}, 2017.

\bibitem{ranshous2015anomaly}
S.~Ranshous, S.~Shen, D.~Koutra, S.~Harenberg, C.~Faloutsos, and N.~Samatova,
  ``Anomaly detection in dynamic networks: a survey,'' \emph{Wiley
  Interdisciplinary Reviews: Computational Statistics}, vol.~7, no.~3, pp.
  223--247, 2015.

\bibitem{heard2010bayesian}
N.~Heard, D.~Weston, K.~Platanioti, and D.~Hand, ``Bayesian anomaly detection
  methods for social networks,'' \emph{The Annals of Applied Statistics},
  vol.~4, no.~2, pp. 645--662, 2010.

\bibitem{corneli2017multiple}
M.~Corneli, P.~Latouche, and F.~Rossi, ``Multiple change points detection and
  clustering in dynamic networks,'' \emph{Statistics and Computing}, pp. 1--19,
  2017.

\bibitem{cheng2009detection}
H.~Cheng, P.-N. Tan, C.~Potter, and S.~Klooster, ``Detection and
  characterization of anomalies in multivariate time series,'' in
  \emph{\Proceedings of the SIAM \International \Conference on Data Mining},
  2009, pp. 413--424.

\bibitem{gupta2014outlier}
M.~Gupta, J.~Gao, C.~C. Aggarwal, and J.~Han, ``Outlier detection for temporal
  data: A survey,'' \emph{IEEE \Transactions on Knowledge and Data
  Engineering}, vol.~26, no.~9, pp. 2250--2267, 2014.

\bibitem{chandola2009anomaly}
V.~Chandola, A.~Banerjee, and V.~Kumar, ``Anomaly detection: A survey,''
  \emph{Computing Surveys}, vol.~41, no.~3, pp. 15:1--15:58, 2009.

\bibitem{scholkopf2001estimating}
B.~Sch{\"o}lkopf, J.~C. Platt, J.~Shawe-Taylor, A.~J. Smola, and R.~C.
  Williamson, ``Estimating the support of a high-dimensional distribution,''
  \emph{Neural Computation}, vol.~13, no.~7, pp. 1443--1471, 2001.

\bibitem{scott2006learning}
C.~D. Scott and R.~D. Nowak, ``Learning minimum volume sets,'' \emph{\Journal
  of Machine Learning Research}, vol.~7, no. Apr, pp. 665--704, 2006.

\bibitem{DI20101910}
J.~Di and E.~Kolaczyk, ``Complexity-penalized estimation of minimum volume sets
  for dependent data,'' \emph{\Journal of Multivariate Analysis}, vol. 101,
  no.~9, pp. 1910--1926, 2010.

\bibitem{breunig2000lof}
M.~Breunig, H.~Kriegel, R.~Ng, and J.~Sander, ``Lof: identifying density-based
  local outliers,'' in \emph{ACM {SIGMOD} Record}, vol.~29, no.~2, 2000, pp.
  93--104.

\bibitem{huang2007network}
L.~Huang, X.~Nguyen, M.~Garofalakis, M.~Jordan, A.~Joseph, and N.~Taft,
  ``In-network {PCA} and anomaly detection,'' in \emph{Advances in Neural
  Information Processing Systems}, 2007, pp. 617--624.

\bibitem{priebe2005scan}
C.~E. Priebe, J.~M. Conroy, D.~J. Marchette, and Y.~Park, ``Scan statistics on
  {E}nron graphs,'' \emph{Computational \& Mathematical Organization Theory},
  vol.~11, no.~3, pp. 229--247, 2005.

\bibitem{wang2014locality}
H.~Wang, M.~Tang, Y.~Park, and C.~Priebe, ``Locality statistics for anomaly
  detection in time series of graphs,'' \emph{IEEE \Transactions on Signal
  Processing}, vol.~62, no.~3, pp. 703--717, 2014.

\bibitem{peel2015detecting}
L.~Peel and A.~Clauset, ``Detecting change points in the large-scale structure
  of evolving networks,'' in \emph{\Proceedings of the AAAI \Conference on
  Artificial Intelligence}, vol.~15, 2015, pp. 1--11.

\bibitem{aggarwal2011outlier}
C.~C. Aggarwal, Y.~Zhao, and S.~Y. Philip, ``Outlier detection in graph
  streams,'' in \emph{\Proceedings of the IEEE \International \Conference on
  Data Engineering}, 2011, pp. 399--409.

\bibitem{neil2013scan}
J.~Neil, C.~Hash, A.~Brugh, M.~Fisk, and C.~Storlie, ``Scan statistics for the
  online detection of locally anomalous subgraphs,'' \emph{Technometrics},
  vol.~55, no.~4, pp. 403--414, 2013.

\bibitem{sun2007less}
J.~Sun, Y.~Xie, H.~Zhang, and C.~Faloutsos, ``Less is more: Compact matrix
  decomposition for large sparse graphs,'' in \emph{\Proceedings of the SIAM
  \International \Conference on Data Mining}, 2007, pp. 366--377.

\bibitem{kolda2008scalable}
T.~G. Kolda and J.~Sun, ``Scalable tensor decompositions for multi-aspect data
  mining,'' in \emph{\Proceedings of the IEEE \International \Conference on
  Data Mining}, 2008, pp. 363--372.

\bibitem{ji2013incremental}
T.~Ji, D.~Yang, and J.~Gao, ``Incremental local evolutionary outlier detection
  for dynamic social networks,'' in \emph{\Proceedings of the Joint European
  \Conference on Machine Learning and Knowledge Discovery in Databases}.\hskip
  1em plus 0.5em minus 0.4em\relax Springer, 2013, pp. 1--15.

\bibitem{pincombe2005anomaly}
B.~Pincombe, ``Anomaly detection in time series of graphs using arma
  processes,'' \emph{Asor Bulletin}, vol.~24, no.~4, p.~2, 2005.

\bibitem{wan2009link}
X.~Wan, E.~Milios, N.~Kalyaniwalla, and J.~Janssen, ``Link-based event
  detection in email communication networks,'' in \emph{\Proceedings of the ACM
  \Symposium on Applied Computing}, 2009, pp. 1506--1510.

\bibitem{boucheron2013concentration}
S.~Boucheron, G.~Lugosi, and P.~Massart, \emph{Concentration inequalities: A
  nonasymptotic theory of independence}.\hskip 1em plus 0.5em minus 0.4em\relax
  Oxford University Press, 2013.

\bibitem{hoeffding1963probability}
W.~Hoeffding, ``Probability inequalities for sums of bounded random
  variables,'' \emph{\Journal of the American Statistical Association},
  vol.~58, no. 301, pp. 13--30, 1963.

\bibitem{gyorfi2006distribution}
L.~Gy{\"o}rfi, M.~Kohler, A.~Krzyzak, and H.~Walk, \emph{A distribution-free
  theory of nonparametric regression}.\hskip 1em plus 0.5em minus 0.4em\relax
  Springer Science \& Business Media, 2006.

\bibitem{gyorfi2002principles}
L.~Gy{\"o}rfi, \emph{Principles of nonparametric learning}.\hskip 1em plus
  0.5em minus 0.4em\relax Springer, 2002, vol. 434.

\bibitem{breiman2001random}
L.~Breiman, ``Random forests,'' \emph{Machine learning}, vol.~45, no.~1, pp.
  5--32, 2001.

\bibitem{scikit-learn}
F.~Pedregosa, G.~Varoquaux, A.~Gramfort, V.~Michel, B.~Thirion, O.~Grisel,
  M.~Blondel, P.~Prettenhofer, R.~Weiss, V.~Dubourg, J.~Vanderplas, A.~Passos,
  D.~Cournapeau, M.~Brucher, M.~Perrot, and E.~Duchesnay, ``Scikit-learn:
  Machine learning in {P}ython,'' \emph{\Journal of Machine Learning Research},
  vol.~12, pp. 2825--2830, 2011.
\end{thebibliography}
\end{document}